\documentclass{article}

\usepackage{amsmath,amssymb, amsthm}

\DeclareMathOperator{\E}{\mathbb{E}}

\newcommand{\Loss}{\mathcal{L}}
\newcommand{\KL}{\mathit{KL}}
\newcommand{\MI}{\mathit{MI}}

\usepackage{mathtools}
\DeclarePairedDelimiterX{\KLdel}[2]{(}{)}{%
  #1\;\delimsize\|\;#2%
}

\newcommand{\mb}[1]{\mathbf{#1}}

\newtheorem{lemma}{Lemma}[section]

\usepackage{microtype}
\usepackage{graphicx}
\usepackage{subfigure}
\usepackage{booktabs}
\usepackage{caption}

\usepackage{hyperref}

\usepackage[accepted]{icml2019}

\icmltitlerunning{VMI-VAE}

\begin{document}

\twocolumn[
\icmltitle{VMI-VAE: Variational Mutual Information Maximization Framework for VAE With Discrete and Continuous Priors}

\begin{icmlauthorlist}
\icmlauthor{Andriy Serdega}{ai}
\icmlauthor{Dae-Shik Kim}{ee}
\end{icmlauthorlist}

\icmlaffiliation{ai}{Graduate School of AI, KAIST}
\icmlaffiliation{ee}{School of Electrical Engineering, KAIST}

\icmlcorrespondingauthor{Andriy Serdega}{a.serdega@kaist.ac.kr},
\icmlcorrespondingauthor{Dae-Shik Kim}{daeshik@kaist.ac.kr}

\icmlkeywords{Machine Learning}

\vskip 0.3in
]
\printAffiliationsAndNotice{} 

\begin{abstract}
Variational Autoencoder is a scalable method for learning latent variable models of complex data. It employs a clear objective that can be easily optimized. However, it does not explicitly measure the quality of learned representations. We propose a Variational Mutual Information Maximization Framework for VAE to address this issue. It provides an objective that maximizes the mutual information between latent codes and observations. The objective acts as a regularizer that forces VAE to not ignore the latent code and allows one to select particular components of it to be most informative with respect to the observations. On top of that, the proposed framework provides a way to evaluate mutual information between latent codes and observations for a fixed VAE model.
\end{abstract}


\section{Introduction}
\label{intro}

Finding a proper data representation can be a crucial part of a given machine learning approach. In many cases, when there is a need for inferring property from a sample, it is the main purpose of the method. For instance, classification task aims to discover such data representation that has a useful high-level interpretation for a human such as class. Unsupervised learning aims to find patterns in unlabeled data that can somehow help to describe it and/or perform a relevant task. Recent deep neural network approaches tackle this problem from a perspective of representation learning, where the goal is to learn a representation that captures some semantic properties of data. If learned representations of salient properties are interpretable and disentangled it would improve generalization and make the downstream tasks robust and easier~\cite{lake2016building}. 

Over the last decade, generative models have become popular in unsupervised learning research. The intuition is that by generative modeling it may be possible to discover latent representations and their relation to observations. The two most popular training frameworks for such models are Generative Adversarial Networks (GAN) \cite{goodfellow2014generative} and Variational Autoencoders \cite{kingma2013auto, rezende2014stochastic}. The latter is a powerful method for unsupervised learning of directed probabilistic latent variable models. Training a model within Variational Autoencoder (VAE) framework allows performing both tasks of inference and generation. This model is trained by maximizing the evidence lower bound (ELBO) which is a clear objective and results in stable training. However, the latent variable in ELBO is marginalized, thus it does not assess the ability of the model to do inference and the quality of latent code \cite{huszar2017maximum,alemi2018fixing}. Therefore, having a high ELBO does not necessarily mean that useful latent representations were learned. Moreover, powerful enough decoder can ignore the conditioning on latent code \cite{bowman2016generating,chen2016variational}.

The key idea of our approach is to maximize mutual information (MI) between samples from the posterior distribution (represented by the encoder) and observations. Unfortunately, exact MI computing is hard and may be intractable. To overcome this, our framework employs Variational Information Maximization \cite{barber2003algorithm} to obtain lower bound on true MI. This technique relies on approximation by auxiliary distribution and we represent it by the additional inference network. The obtained lower bound on MI is used as the regularizer to the original VAE objective to force the latent representations to have a strong relationship with observations and prevent the model from ignoring them. 
 
We have conducted our experiments on VAE models trained on MNIST and FasionMNIST with such latent distributions: Gaussian, and joint Gaussian and discrete. We compare qualitatively and quantitatively the models trained using pure ELBO objective and with introduced MI regularizer.

\section{Related Work}
\label{relatedWork}

There is a number of works that propose approaches to improve latent representations learned by VAE. In \cite{bowman2016generating} authors vary the weight of KL divergence component of objective function during training by gradual increasing it from 0 to 1. $\beta$-VAE~\cite{higgins2017beta} employs weighting coefficient that scales KL divergence term. It balances latent channel capacity and independence constraints with reconstruction accuracy to improve the disentanglement of representations \cite{burgess2018understanding}. \citealt{alemi2018fixing} introduced an information theoretic framework to characterize tradeoff between compression and reconstruction accuracy. Authors use bounds on mutual information to derive a rate-distortion curve. On this curve, different points represent a family of models with the same ELBO but different characteristics. Also, the authors state that the proposed framework generalizes aforementioned $\beta$-VAE in the sense that this coefficient controls MI between observations and latent variables. InfoVAE~\cite{zhao2017infovae} employs a modification of the objective to weight the preference between correct inference and fitting data distribution, and specify preference on how much the model should rely on the latent variables. In~\cite{chen2018isolating} authors decompose ELBO to determine the source of disentanglement in $\beta$-VAE. Additionally, this work introduces $\beta$-TCVAE that is able to discover more interpretable representations. The authors of InfoGAN~\cite{chen2016infogan} address problem of entangled latent representations in GAN by maximizing the mutual information between the part of the latent code and produced by generator samples. For that purpose, they also employ Variational Information Maximization \cite{barber2003algorithm}.
\section{Information Maximization for VAE}
\label{Variational Information Maximization}

\subsection{Variational Autoencoder}
VAE \cite{kingma2013auto, rezende2014stochastic} is a scalable model for unsupervised learning directed probabilistic latent variable models. It is is trained by maximizing ELBO
\begin{multline} \label{eq:beta:base}
    \Loss(\theta, \phi) = \E_{q_\phi(\mb{z}|\mb{x})}[\log p_\theta(\mb{x}|\mb{z})] \\
    - D_{\KL}\KLdel{q_\phi(\mb{z}|\mb{x})}{p(\mb{z})}
\end{multline}
where $q_\phi(\mb{z}|\mb{x})$ is an approximate posterior distribution represented by encoder neural network with parameters $\phi$ and $p_\theta(\mb{x}|\mb{z})$ is decoder network parametrized by $\theta$. By passing samples from prior distribution $p(\mb{z})$ to decoder it is possible to generate new data samples.

\subsection{Variational Mutual Information}
Mutual information (MI) between samples of the posterior distribution and observation is formally defined as
\begin{equation}
\label{information-z-x}
    I(\mb{z}; \mb{x}) = H(\mb{z}) - H(\mb{z} | \mb{x}),
\end{equation}
where $H(\cdot)$ denotes the entropy of the corresponding variables.

$I(\mb{z}, \mb{x})$ is intractable, since it requires the intractable posterior $p(\mb{z} | \mb{x})$ to compute. Following the reasoning in \cite{chen2016infogan}, we obtain the lower bound on MI between observations and latent variables:
\begin{equation}
\begin{split}
    I(\mb{z}; \mb{x}) &= H(\mb{z}) - H(\mb{z} | \mb{x}) \\
    &= \E_{\mb{x} \sim  p_\theta(\mb{x}|\mb{z})} \big[ \E_{\mb{z} \sim q_\phi(\mb{z}|\mb{x})} [\log P(\mb{z} | \mb{x})]\big]  + H(\mb{z})\\
    &= \E_{\mb{x} \sim  p_\theta(\mb{x}|\mb{z})} \big[D_{\KL}\KLdel{P(\cdot|x)}{Q(\cdot | x)} \\
    &\qquad + \E_{\mb{z'} \sim q_\phi(\mb{z}|\mb{x})}[\log Q(\mb{z'}|\mb{x})] \big] + H(\mb{z})\\
    &\geq \E_{\mb{x} \sim  p_\theta(\mb{x}|\mb{z})} \big[\E_{\mb{z'} \sim q_\phi(\mb{z}|\mb{x})}[\log Q(\mb{z'}|\mb{x})]\big] + H(\mb{z}),
\end{split}
\end{equation}
where $Q$ is auxiliary distribution. We represent it by a neural network that takes the decoder output as input. We treat the $H(\mb{z})$ as a constant for simplicity. 

The problem with the obtained lower bound is that it requires sampling from the posterior in the inner expectation. This can be overcome by applying this lemma used in InfoGAN~\cite{chen2016infogan}. The proof could be found in Appendix~\ref{sec:lemmaproof}.

\begin{lemma}
\label{thelemma}
For random variables $X, Y$ and function $f(x, y)$ under suitable regularity conditions: 
\begin{multline*}
    \E_{x \sim X, y \sim  Y|x} [f(x, y)] = \\ 
    \E_{x \sim X, y \sim  Y|x, x' \sim  X|y} [f(x', y)].
\end{multline*}
\end{lemma}

With this we could re-define the variational lower bound on MI: 
\begin{multline}
\label{eq:mut_information_var_lower_bound}
    \E_{\mb{x} \sim  p_\theta(\mb{x}|\mb{z})} \big[\E_{\mb{z'} \sim q_\phi(\mb{z}|\mb{x})}[\log Q(\mb{z'}|\mb{x})]\big] + H(\mb{z}) \\
    = \E_{\mb{z} \sim q_\phi(\mb{z}|\mb{x}), \mb{x} \sim  p_\theta(\mb{x}|\mb{z})} [\log Q(\mb{z}|\mb{x})] + H(\mb{z}) \\
    \leq I(\mb{z};\mb{x})
\end{multline}

By using this lower bound for a fixed VAE it is possible now to maximize mutual information between observations and latent variables by maximizing the lower bound.
\begin{equation} \label{eq:mut_information_bound}
    \max_{Q} \ \E_{\mb{z} \sim q_\phi(\mb{z}|\mb{x}), \mb{x} \sim  p_\theta(\mb{x}|\mb{z})} [\log Q(\mb{z}|\mb{x})] + H(\mb{z}) \leq I(\mb{z};\mb{x}),
\end{equation}

Finally, we define mutual information maximization regularizer $\MI$ for variational autoencoder as \begin{multline} \label{eq:mut_information_reg}
    \MI(\theta, \phi, Q) = \\ 
    \E_{\mb{z} \sim q_\phi(\mb{z}|\mb{x}), \mb{x} \sim  p_\theta(\mb{x}|\mb{z})} [\log Q(\mb{z}|\mb{x})] + H(\mb{z}) 
\end{multline}
that can be estimated using Monte Carlo sampling.

\subsection{Resulting Framework}
In our Variational Mutual Information Maximization Framework we combine ELBO with the proposed $\MI$ regularizer to form the objective
\begin{equation} \label{eq:newmaxmax}
    \max_{\theta, \phi} \max_{Q} \  \Loss(\theta, \phi) + \lambda \MI(\theta, \phi, Q)
\end{equation}
where $\lambda$ is a scaling coefficient that controls the impact of $\MI$ on VAE training. For each training batch, we maximize the objective with respect to the auxiliary distribution $Q$ to make lower bound on mutual information tighter first. Then, we maximize it with respect to parameters of the VAE ($\theta$ and $\phi$) to train it using the $\MI$ regularizer to maximize MI between latent codes $\mb{z}$ and observations $\mb{x}$. Please, see Fig.1 for visualization of the proposed model.
\begin{figure}
\centering
\begin{center}
\centerline{\includegraphics[width=1\linewidth]{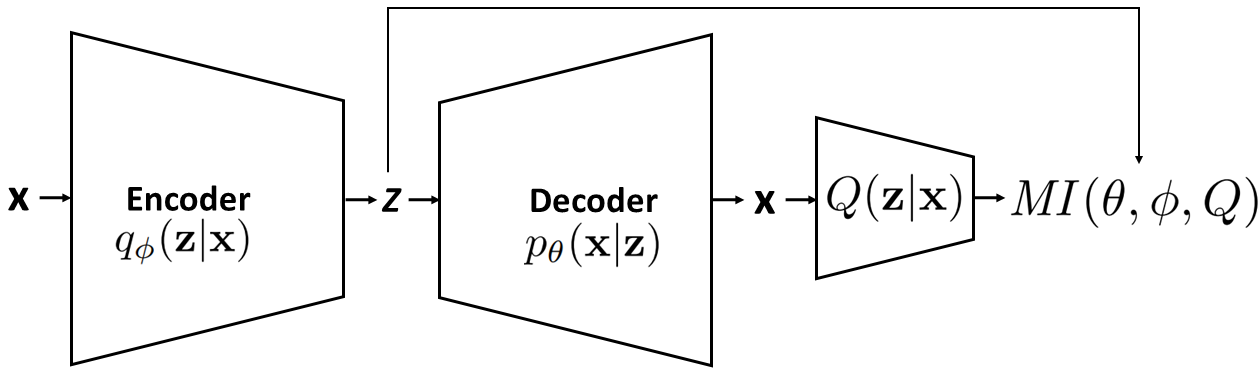}}
\caption{Structure of proposed model}
\label{fig:VIM-VAE}
\end{center}
\vspace{-4mm}
\end{figure}
\section{Experimental Setup}
\label{Experimental Setup}
\subsection{VAE with Gaussian latent}
In this setting, we employ VAE with 32-dimensional Gaussian latent variable with prior $p(\mb{z})=\mathcal{N}(0,1)$. We train and compare two identically initialized networks with same hyperparameters on MNIST. One is trained using only ELBO objective (1) and the other with mutual information maximization (4). For the latter case, we select only two components latent code vector forming sub-vector ($\mb{z}_{1}$,$\mb{z}_{2}$)=$\mb{\hat{z}}$ for mutual information maximization and define the regularizer as
\begin{multline} \label{eq:mut_information_contin}
    \MI(\theta, \phi, Q) = \\ 
    \E_{\mb{\hat{z}} \sim q_\phi(\mb{\hat{z}}|\mb{x}), \mb{x} \sim  p_\theta(\mb{x}|\mb{{z}})} [\log Q(\mb{\hat{z}}|\mb{x})] + H(\mb{\hat{z}})
\end{multline}
We select only two components of the latent code since it is straightforward to illustrate their impact on observations in 2D visualizations by just manipulating 
their individual values without any latent space interpolations.
\subsection{VAE with joint Gaussian and discrete latent}
In this section, we define setting for VAE model with joint latent distribution of continuous and discrete (categorical) variables. We define $\mb{z}$ as 16-dimensional Gaussian part of latent code with prior $p(\mb{z})=\mathcal{N}(0,1)$ and  $\mb{c}$ as discrete part with 10 categories and uniform prior. In this setting, the encoder network represents joint posterior approximation $q_\phi(\mb{z},\mb{c}|\mb{x})$, decoder network is $p_\theta(\mb{x}|\mb{z}, \mb{c})$ and prior is $p(\mb{z}, \mb{c})$. Then, the resulting ELBO objective has the form of
\begin{multline} \label{eq:elbo:with_discrete}
    \Loss(\theta, \phi) = \E_{q_\phi(\mb{z}, \mb{c} | \mb{x})}[\log p_\theta(\mb{x}|\mb{z}, \mb{c})] \\
    - D_{\KL}\KLdel{q_\phi(\mb{z}, \mb{c}|\mb{x})}{p(\mb{z}, \mb{c})}
\end{multline}
By the assumption that $\mb{c}$ and $\mb{z}$ are mutually and conditionally independent, we can decompose the $D_{\KL}$ term as
\begin{multline} \label{eq:KL_with_discrete}
    D_{\KL}\KLdel{q_\phi(\mb{z}, \mb{c}|\mb{x})}{p(\mb{z}, \mb{c})} = \\
    D_{\KL}\KLdel{q_\phi(\mb{z}|\mb{x})}{p(\mb{z})}
    + D_{\KL}\KLdel{q_\phi(\mb{c}|\mb{x})}{p(\mb{c})}
\end{multline}
For the categorical latent variable, we employ a continuous differentiable relaxation technique proposed by \cite{jang2016categorical,maddison2016concrete}. The categorical variable $\mb{c}$ in our setting has 10 categories and let $\pi_1$...$\pi_{10}$ be the respective probabilities that define this distribution. We represent categorical samples as 10-dimensional one-hot vectors. Also, let $g_{1}$...$g_{10}$ be i.i.d. Gumbel$(0,1)$ samples. Then, using softmax function we can draw samples from this categorical distribution having for 10-dimensional sample vector each component defined as 
\begin{equation} \label{eq:gumbel-softmax-sample}
    y_i = 
    \frac{\exp{((\log(\pi_i) + g_i) / \tau)}}
    {\sum_{j = 1}^{10} \exp{((\log(\pi_j) + g_j)/ \tau)} }
\end{equation}
for $i = 1, \dots, 10$. Where $\tau$ is temperature hyperparameter.
For this VAE form trained with mutual information maximization, we maximize MI with respect to observation $\mb{x}$ and categorical latent variable $\mb{c}$. In that case, we define our mutual information maximization regularizer term as 
\begin{multline} \label{eq:mut_information_discrete}
    \MI(\theta, \phi, Q) = \\ 
    \E_{\mb{c} \sim q_\phi(\mb{c}|\mb{x}), \mb{x} \sim  p_\theta(\mb{x}|\mb{z},\mb{c})} [\log Q(\mb{c}|\mb{x})] + H(\mb{c})
\end{multline}
In this setting, we train and compare two identically initialized VAE models with same hyperparameters on MNIST and FasionMNIST. One using only objective~(eq.~\ref{eq:elbo:with_discrete}) and other combined with $\MI$ regularizer following~(eq.~\ref{eq:newmaxmax}).

\section{Experimental results}
\label{Experimental results}
\subsection{VAE with Gaussian latent}
As we mentioned before, we trained two identically initialized VAE models: one using ELBO objective and one with added $\MI$ regularizer for sub-part of latent code ($\mb{z}_{1}$,$\mb{z}_{2}$)=$\mb{\hat{z}}$. In Fig.2 we provide a qualitative comparison of the impact of these two components of the code on produced samples. For each latent code, we vary $\mb{z}_{1}$ and $\mb{z}_{2}$ from -3 to 3 with fixed remaining part and decode it. As you can see in Fig.2 (a), $\mb{\hat{z}}$ in vanilla VAE does not have much impact on the output samples. In contrast, $\mb{\hat{z}}$ with maximized mutual information in VAE by $\MI$ regularizer have a significant impact on output samples. For this model, we can see that outputs morph between three digit types as the code changes. Moreover, you can see that the particular combinations of these two components of 32-dimensional code morph the original sample into digit 1 and 6 regardless of the original sample type. All of this means that the provided regularizer indeed forces this part of learned latent codes $\mb{\hat{z}}$ to have high MI and strong relationship with observations.
\begin{figure}[ht]
\centering 
\subfigure[VAE]{\includegraphics[width=80mm]{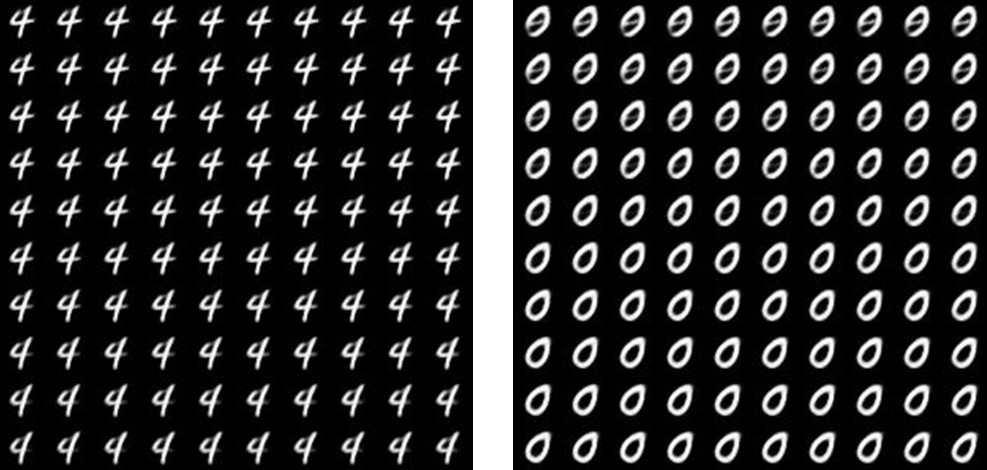}}
\subfigure[VAE with MI maximization]{\includegraphics[width=80mm]{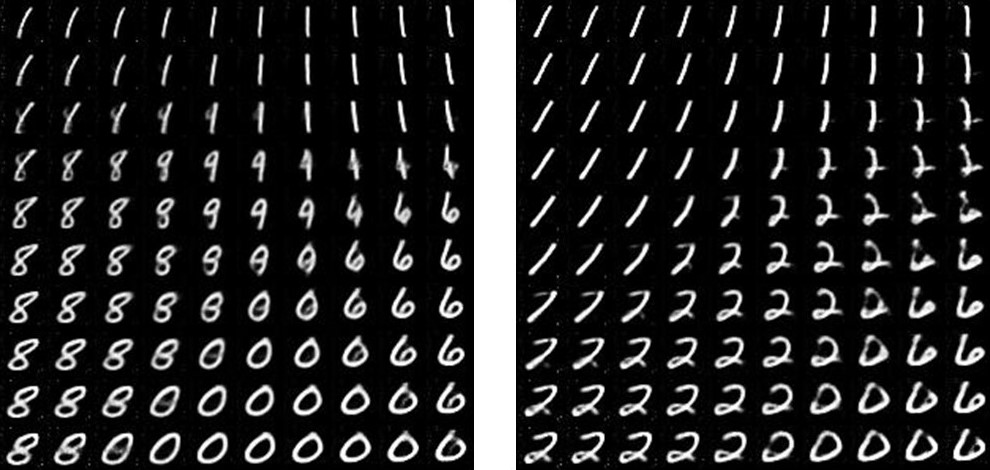}}
\caption{Latent code manipulations of samples from VAE with Gaussian latent variable: (a) trained using only ELBO objective, (b) trained with MI maximization. We vary each component of ($\mb{z}_{1}$,$\mb{z}_{2}$)=$\mb{\hat{z}}$ from -3 to 3 having the remaining part of the code fixed. Rows and columns represent values of $\mb{z}_{1}$ and $\mb{z}_{2}$ respectively.}
\label{fig:mnist_traversals}
\end{figure}
\begin{figure}[h]
\centering
\includegraphics[width=0.8\linewidth]{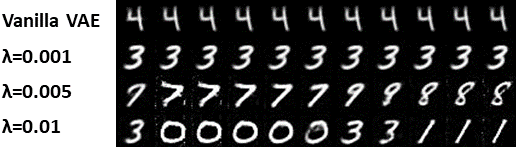}
\caption{Comparison of $\mb{z}_{1}$ impact change on the output samples with different values of $\MI$ weight coefficient $\lambda$. The first column represents the original sample and each row represents varied $\mb{z}_{1}$ from -3 to 3}
\label{fig:gvae}
\end{figure}

Also, we compare resulting models with different values of scaling coefficient of $\MI$ regularizer $\lambda$ in Fig.3. As you can see, with a low value of lambda, the impact of $\mb{z}_{1}$ is the same as in vanilla VAE. However, with the increase of $\lambda$, the impact of $\mb{z}_{1}$\ on observations also increases.
\subsection{VAE with joint Gaussian and discrete latent}
In this section, we compare two identically initialized VAE model with joint Gaussian and discrete latent variables but trained in a different manner. One model is trained using only the ELBO of the form represented by eq.\ref{eq:elbo:with_discrete}. The second one is trained with added $\MI$ regularizer (eq.\ref{eq:mut_information_discrete}) for MI maximization between data samples and categorical part of the learned latent code.
\begin{figure}[h]
\centering 
\subfigure[VAE]{\includegraphics[width=45mm]{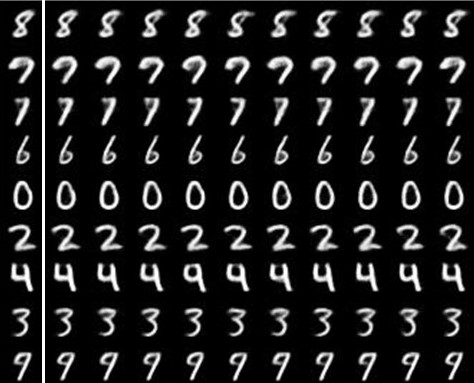}}
\subfigure[VAE with MI maximization]{\includegraphics[width=50mm]{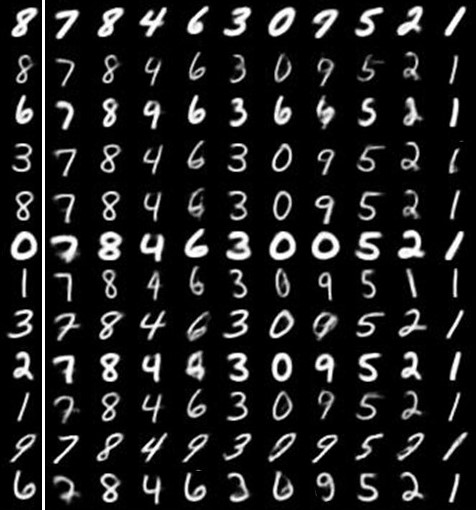}}
\caption{Latent code manipulations of samples from VAE with joint Gaussian and discrete latent variable: (a) trained with only ELBO objective, (b) with $\MI$ regularizer. The first (separated) column represents the original samples. The following rows represent this samples with changed categorical part of latent code between 10 categories with fixed Gaussian component. Also, each (not separated) column can be seen as generated samples that are conditioned on a particular category with varied Gaussian component}
\end{figure}
\begin{figure}[h]
\centering 
\subfigure[VAE]{\includegraphics[width=45mm]{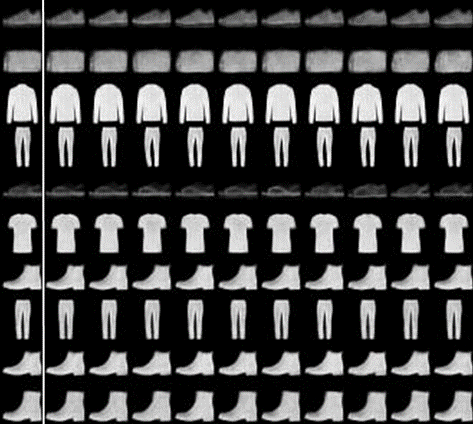}}
\subfigure[VAE with MI maximization]{\includegraphics[width=50mm]{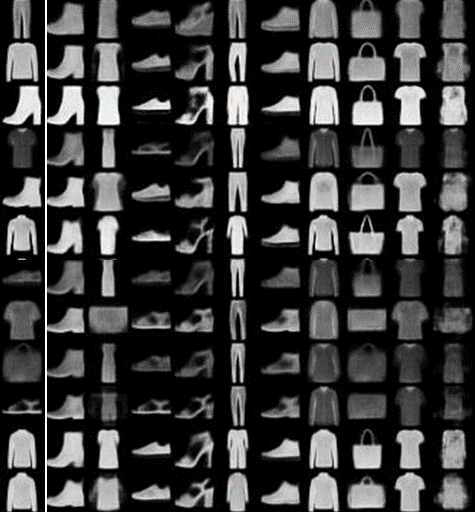}}
\caption{Latent categorical code manipulations of samples from VAE trained on FashionMNIST with joint Gaussian and discrete latent variable: (a) trained with only ELBO objective, (b) with $\MI$ regularizer. Rows represent samples with changed part of categorical latent code.}
\label{fig:fashion_traversals}
\end{figure}
As you can see on Fig.4 (a), in VAE that was trained using pure ELBO, the categorical part of the latent code does not have an influence on produced samples. Thus, even when our strong prior assumption that the data have 10 categories was incorporated into the latent variable, the trained model ignores it and does not assign any interpretable representation to a categorical variable. 

In contrast, for VAE that was trained with MI maximization between observations and categorical code, produced samples show a completely different response to the latent categorical variable change. As the categorical part of latent code varies between 10 categories, the samples change in a class-wise manner. For most of the samples, the particular value of the categorical variable changes them to the same digit type while preserving other features of original sample like thickness and angle. We interpret it as that the model generalizes and disentangles digit type from style representations by categorical and Gaussian part of the latent code respectively. 

For the sake of quantitative comparison, we applied the encoder categorical component as a classifier to the MNIST classification task. VAE trained with ELBO objective has 21\% classification accuracy on MNIST while VAE with $\MI$ regularizer achieved 74\% accuracy.

In Fig.\ref{fig:histograms} we compare histograms of categorical latent variable probabilities that were collected during the training of both models. As you can see, for the case of vanilla VAE the probabilities are mostly concentrated around 0.1 and do not reach the area around 1. In the case of VAE with maximized MI for the categorical variable, you can see that the probability values are concentrated around 0.5 and 0.9. It is natural behavior since when the category probabilities are uniform regardless of the input, it is pointless to do any further inference using this variable for the decoder network and thus the resulting model ignores it. Therefore, we interpret this observation as that VAE model with maximized MI between data samples and categorical part of the latent variable indeed makes more use of this variable.
\begin{figure}[h]
\centering 
\subfigure[VAE]{\includegraphics[width=55mm]{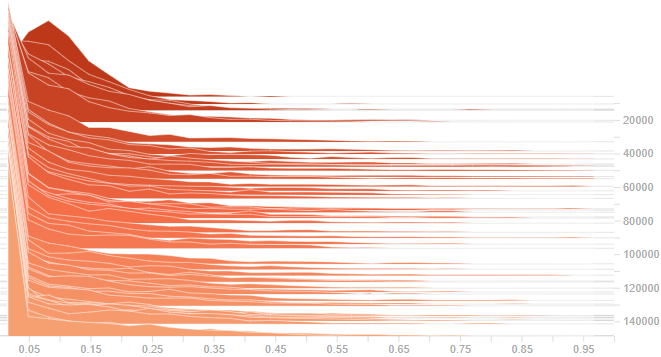}}
\subfigure[VAE with MI maximization]{\includegraphics[width=60mm]{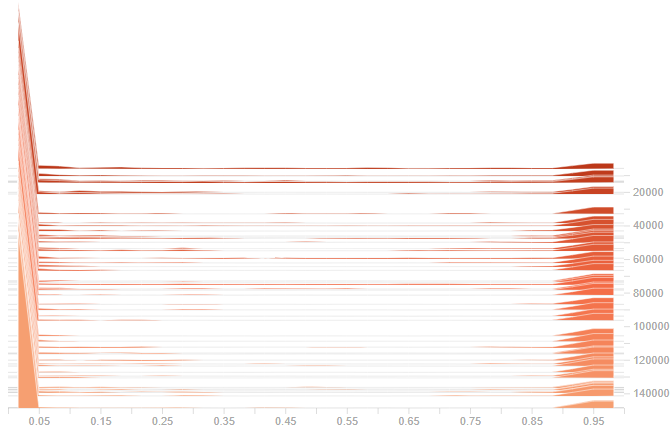}}
\caption{Histograms collected from categorical latent variable probabilities of VAE with joint Gaussian and discrete latent variable.}
\label{fig:histograms}
\end{figure}

As we mentioned before, our Variational Mutual Information Maximization Framework can be used for MI evaluation between latent variables and observations for a fixed VAE by obtaining lover bound on MI (eg.\ref{eq:mut_information_bound}). In Fig.6, we provide plots of lover bound MI estimate between observations and categorical part of latent codes during the training process of two models. One was trained using only the ELBO without MI loss minimization and the other with MI maximization. As you can see, the VAE model with MI maximization has higher MI lover bound estimate during training than one that was trained without $\MI$ regularizer.

\begin{figure}[h]
\centering 
\includegraphics[width=1\linewidth]{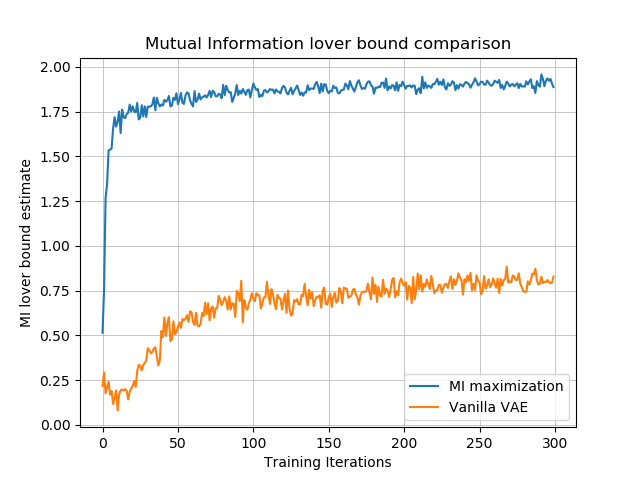}
\caption{Categorical MI lover bound estimate of VAE models (joint Gaussian and discrete latent) during the training with $\MI$ regularizer and without.}
\end{figure}

Also, we provide categorical distribution KL divergence estimate plots in Figure \ref{fig:kl_comparison} for models with and without MI maximization. From the theoretical perspective on KL divergence estimate, the value of this estimate is an upper bound on MI between latent variables and observations. Our experimental results are consistent with it. Moreover, the upper bound for the model with maximized MI reaches the maximum possible value of MI for the categorical variable that is modeled as a discrete distribution with 10 values and uniform prior. For this case, the maximum possible value of MI is the entropy of 10 category uniform distribution that is approximately 2.3. The KL divergence estimate is pretty close to this value for the model with MI maximization for the categorical latent variable as you can see on the plot.

\begin{figure}[H]
\centering
\includegraphics[width=1\linewidth]{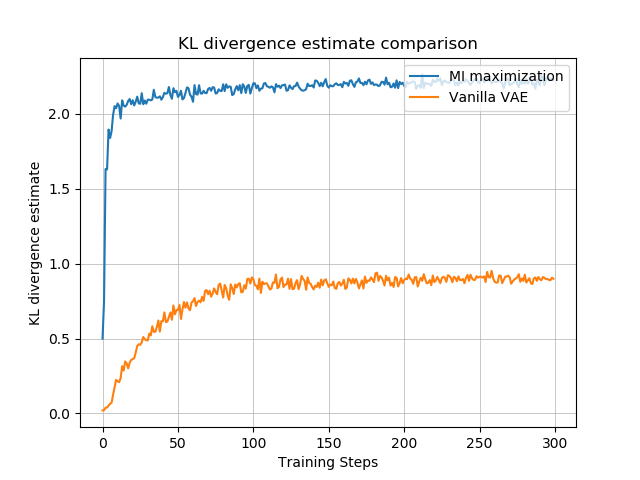}
\caption{KL divergence estimate for categorical latent variable during the process of training VAE models without MI maximization and with $\MI$ regularizer.}
\label{fig:kl_comparison}
\end{figure}
\section{Intuition}
The $Q$ network that represents auxiliary distribution can be seen as a classifier network or a feature extractor network. When we maximize $\MI$ (eq.~\ref{eq:mut_information_reg}) regularizer with respect to parameters of this network, the maximization of the first term is the same as minimization of negative log-likelihood as we do when train classification models. Thus, we can interpret the whole training procedure as training the Q network to correctly classify $\mb{x}$ in terms of its original generative factors $\mb{z}$ or to better extract them from the sample. Then, when maximizing $\MI$ w.r.t. VAE, we are forcing the model to make these features more extractable from the produced sample and classifiable.
\section{Discussion and Conclusion}
\label{Discussion}
In our work, we have presented a method for evaluation, control, and maximization of mutual information between latent variables and observations in VAE. In comparison to other related works, it provides an explicit and tractable objective. Using the proposed technique, it is possible to compare MI for different fixed VAE networks. Moreover, our experimental results illustrate that the Variational Mutual Information Maximization Framework can indeed strengthen the relationship between latent variables and observations. Also, it improves learned representations (section 5.2). However, it comes with an increase in computational and memory cost, since mutual information lover bound estimate requires auxiliary distribution $Q$ that we represent by an additional encoder neural network and train. 

We believe, that our work (with further analysis and improvements) have the potential to fill the gaps between previous theoretical insights for VAE from Information Theory perspective and empirical results. KL divergence term in VAE, by analysis from \cite{makhzani2017pixelgan,kim2018disentangling}, is an upper bound on true mutual information between latent codes and observations. Our empirical results are consistent with this insight: KL divergence estimate for categorical latent (which is 0.89) at the end of training in VAE without MI maximization is less than for VAE with MI maximization (which is 2.22). Please, see Fig.5 for KL divergence estimates values collected during training for both models.

On top of that, in \cite{alemi2018fixing} authors state that in $\beta$-VAE, by varying $\beta$ coefficient that scales KL divergence it is possible to vary $I(\mb{x};\mb{z})$. In \cite{burgess2018understanding} along with KL scaling coefficient, authors propose to use the constant that KL divergence estimate should match and thus explicitly control mutual information upper bound. 

In \cite{dupont2018learning} authors tackle the problem of training VAE model with joint Gaussian and categorical latent that is similar to those that we define in section 4.2. They reported that the model even trained in $\beta$-VAE setting also ignores the categorical variable. Thus, they followed weighting and constraining technique by \cite{burgess2018understanding} for each KL divergence term to force VAE to not ignore this latent part. In our work, we tackle the same problem from different perspectives. They increase and control the upper bound on MI. In contrast, we maximize lower bound on MI. Which approach is more suitable for a particular problem is an open question and possible future research direction. However, our explicit formulation of MI maximization, that lead to similar results, bridges theoretical insights that KL divergence controls MI between observations and latent codes.

\bibliography{main}
\bibliographystyle{icml2019}

\onecolumn

\appendix
\section{Proof of the Lemma~\ref{thelemma}}
\label{sec:lemmaproof}

\begin{lemma}
\label{thelemma_appendix}
For random variables $X, Y$ and function $f(x, y)$ under suitable regularity conditions: 
\begin{equation*}
    \E_{x \sim X, y \sim  Y|x} [f(x, y)] = \\ 
    \E_{x \sim X, y \sim  Y|x, x' \sim  X|y} [f(x', y)].
\end{equation*}
\end{lemma}

\begin{proof}
This proof was originally introduced in~\cite{Ford2018lemmaproof}.

Make expectations explicit: 
\begin{equation*}
        \E_{x \sim X, y \sim Y|x}[f(x, y)] = \
        E_{x \sim P(X)}\big[\E_{y \sim P(Y|X=x)}[f(x, y)]\big]
\end{equation*}

By definition of $P(Y|X=x)$ and $P(X|Y=y)$:
\begin{equation*}
        \E_{x \sim P(X)}\big[\E_{y \sim P(Y|X=x)}[f(x, y)]\big] = 
        \E_{x,y \sim P(X,Y)}[f(x, y)] = \E_{y \sim P(Y)}\big[\E_{x \sim P(X|Y=y)}[f(x, y)]\big]
\end{equation*}

Rename $x$ to $x'$:
\begin{equation*}
        \E_{y \sim P(Y)}\big[\E_{x \sim P(X|Y=y)}[f(x, y)]\big] = 
        \E_{y \sim P(Y)}\big[\E_{x' \sim P(X|Y=y)}[f(x', y)]\big]
\end{equation*}

By the law of total expectation:
\begin{equation*}
        \E_{y \sim P(Y)}\big[\E_{x' \sim P(X|Y=y)}[f(x', y)]\big] = 
        \E_{x \sim P(X)}\Big[\E_{y \sim P(Y|X=x)}\big[\E_{x' \sim P(X|Y=y)}[f(x', y)]\big]\Big]
\end{equation*}

Make expectations implicit:
\begin{equation*}
        \E_{x \sim P(X)}\Big[\E_{y \sim P(Y|X=x)}\big[\E_{x' \sim P(X|Y=y)}[f(x', y)]\big]\Big] = 
        \E_{x \sim X,y \sim Y|x,x' \sim X|y}[f(x', y)] 
\end{equation*}
\end{proof}

\section{Histograms of encoded digits into particular one-hot categorical vector}
\label{sec:collected_histograms}
We have counted numbers of particular digits from MNIST dataset encoded into particular one-hot vectors (categorical variable) for VAE models trained with and without MI maximization. We represent this results in figures \ref{fig:hists_mi} and \ref{fig:hists_nomi}. As you can see, the digit images from particular classes align pretty well with particular one-hot vectors in case of VAE with MI maximization. In contrast, for VAE without MI maximization, the distribution of particular type digit images are uniform across all one-hot vectors.

\begin{figure}[h]
\centering 
\subfigure[One-hot 1]{\includegraphics[width=55mm]{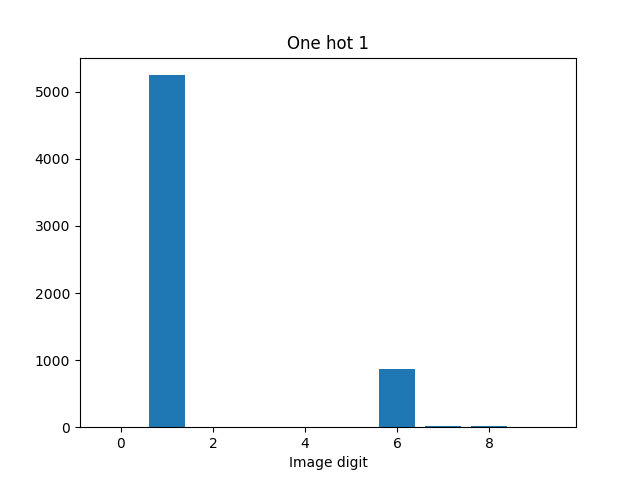}}
\subfigure[One-hot 2]{\includegraphics[width=55mm]{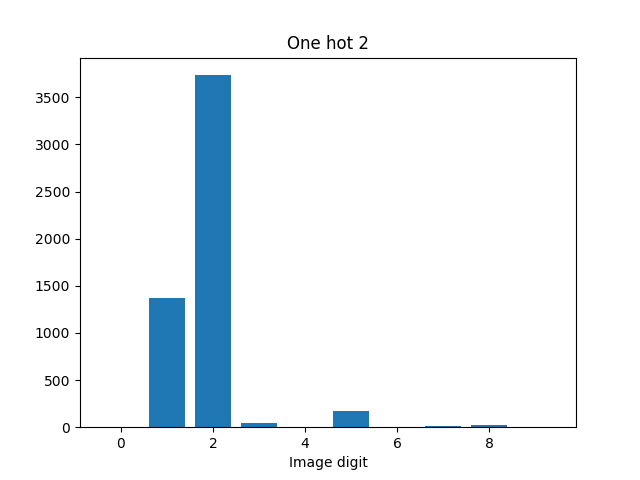}}
\subfigure[One-hot 3]{\includegraphics[width=55mm]{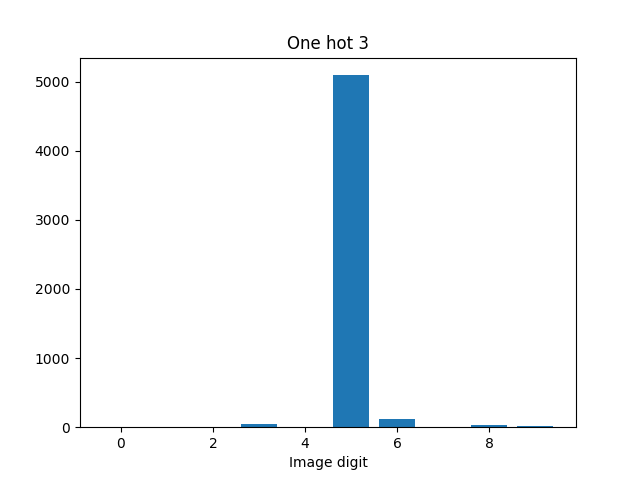}}
\subfigure[One-hot 4]{\includegraphics[width=55mm]{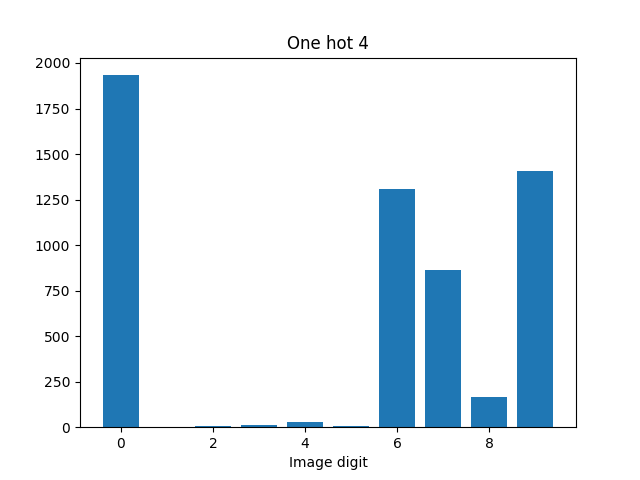}}
\subfigure[One-hot 5]{\includegraphics[width=55mm]{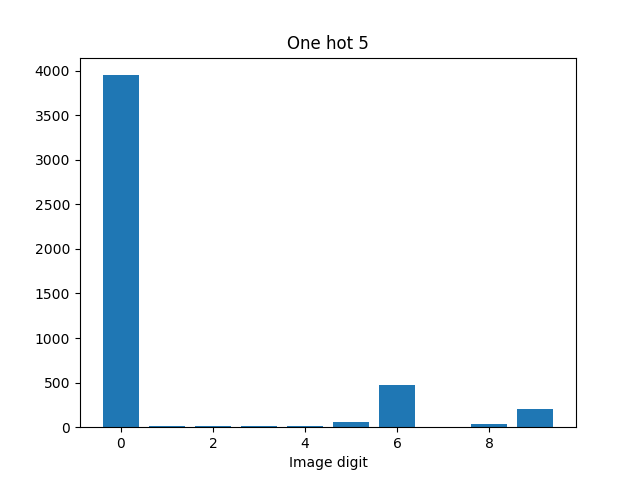}}
\subfigure[One-hot 6]{\includegraphics[width=55mm]{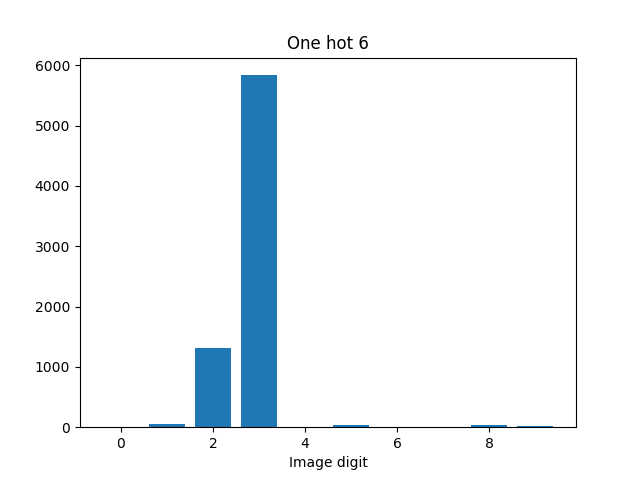}}
\subfigure[One-hot 7]{\includegraphics[width=55mm]{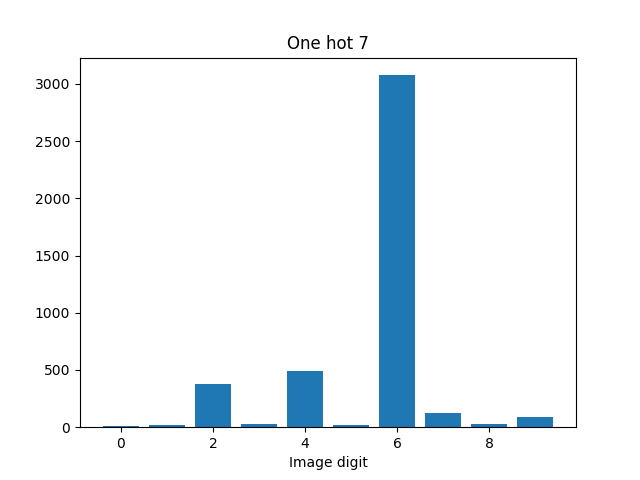}}
\subfigure[One-hot 8]{\includegraphics[width=55mm]{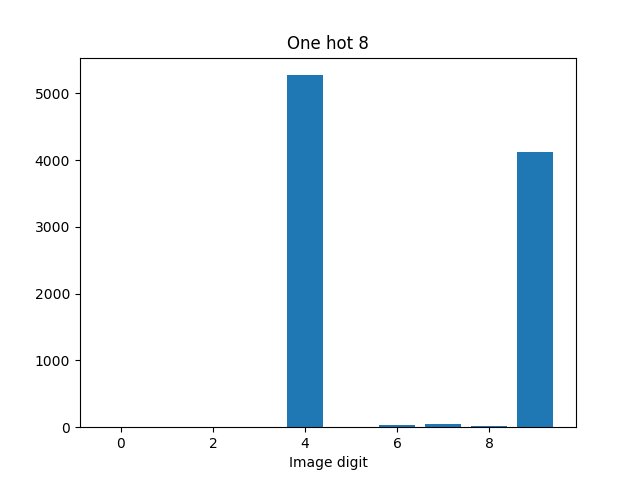}}
\subfigure[One-hot 9]{\includegraphics[width=55mm]{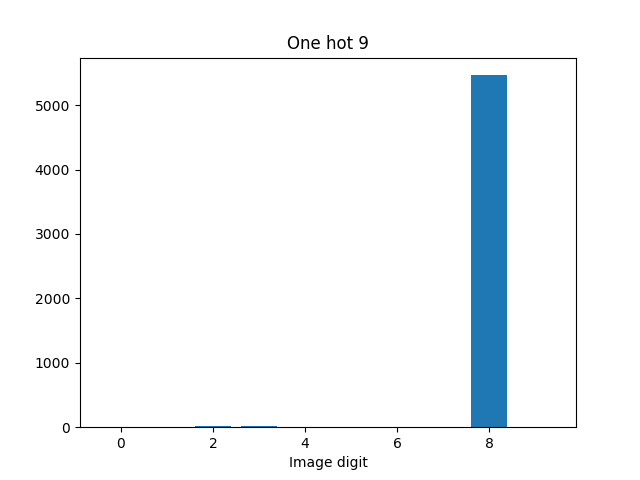}}
\subfigure[One-hot 10]{\includegraphics[width=55mm]{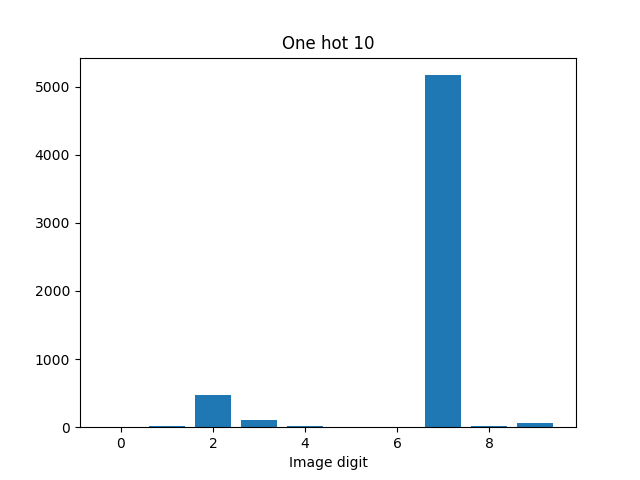}}
\caption{Histograms of numbers of digits encoded into one-hot vectors (that represent categorical latent variable). VAE with MI maximization. }
\label{fig:hists_mi}
\end{figure}

\begin{figure}[h]
\centering 
\subfigure[One-hot 1]{\includegraphics[width=55mm]{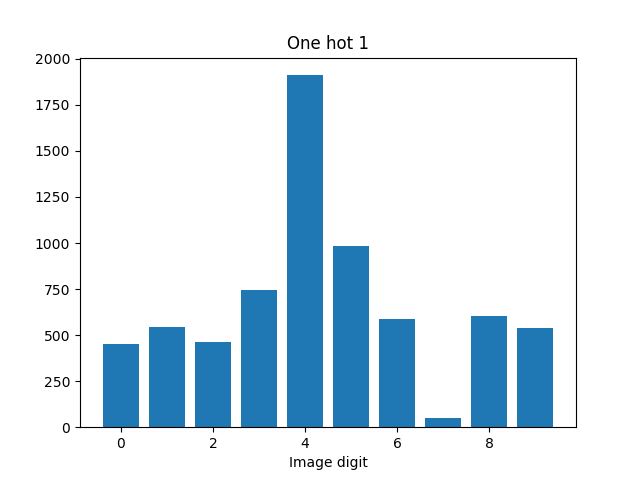}}
\subfigure[One-hot 2]{\includegraphics[width=55mm]{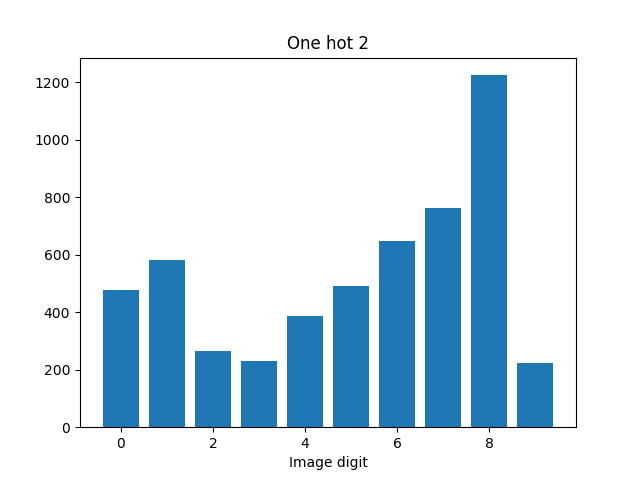}}
\subfigure[One-hot 3]{\includegraphics[width=55mm]{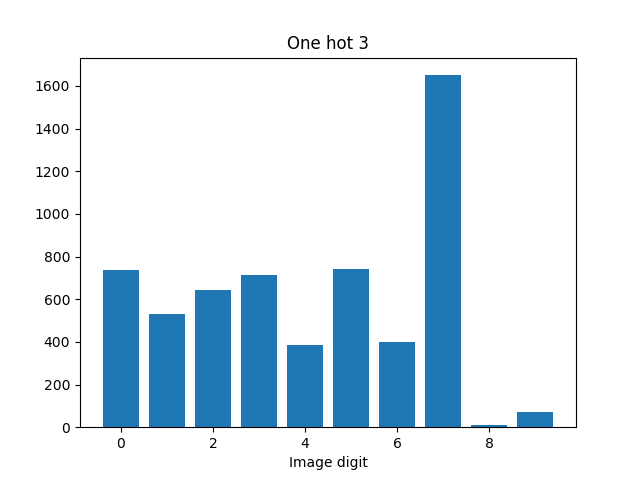}}
\subfigure[One-hot 4]{\includegraphics[width=55mm]{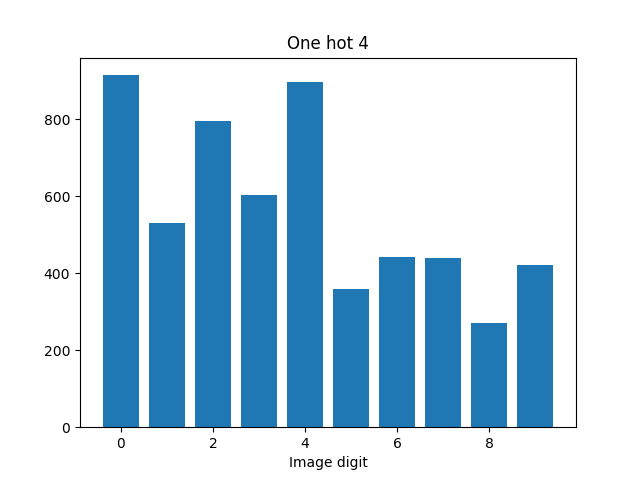}}
\subfigure[One-hot 5]{\includegraphics[width=55mm]{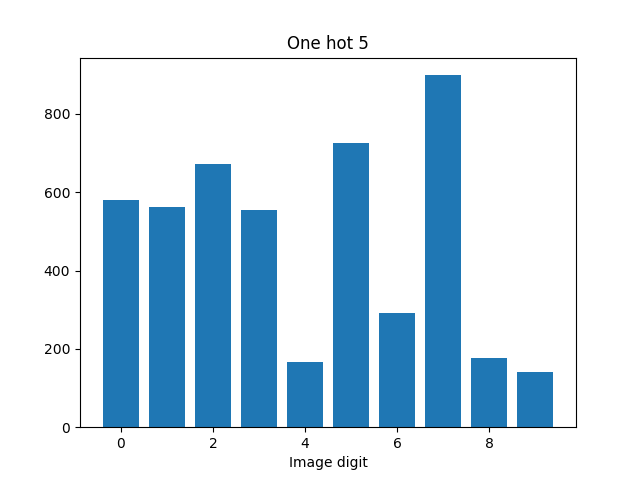}}
\subfigure[One-hot 6]{\includegraphics[width=55mm]{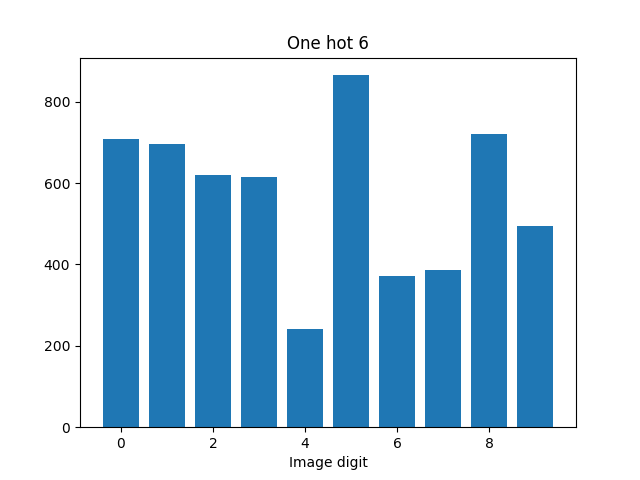}}
\subfigure[One-hot 7]{\includegraphics[width=55mm]{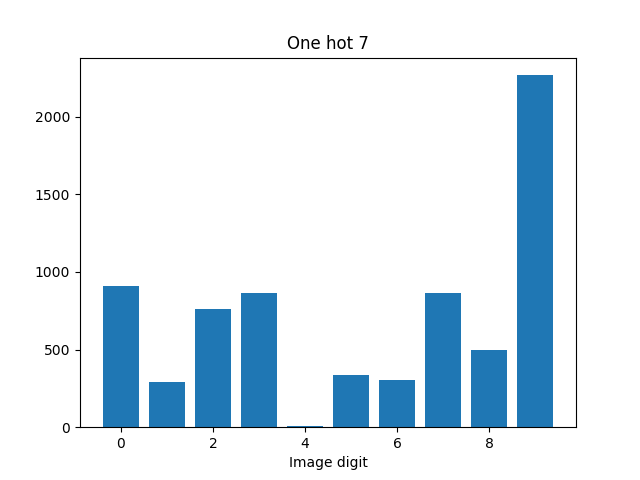}}
\subfigure[One-hot 8]{\includegraphics[width=55mm]{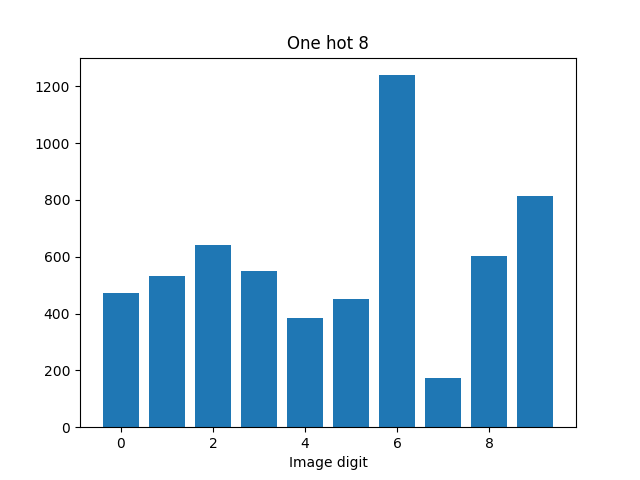}}
\subfigure[One-hot 9]{\includegraphics[width=55mm]{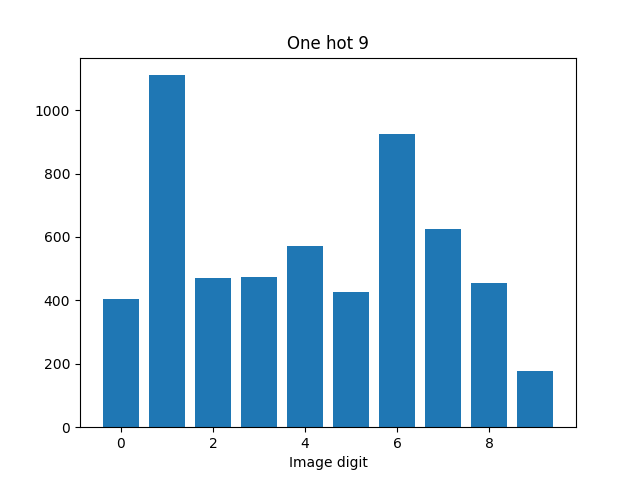}}
\subfigure[One-hot 10]{\includegraphics[width=55mm]{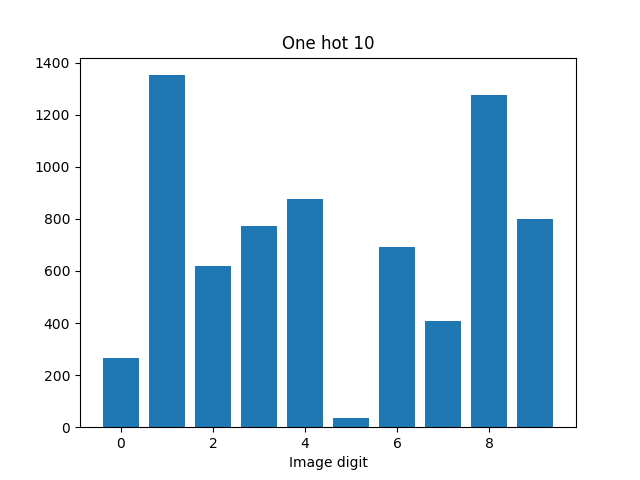}}
\caption{Histograms of numbers of digits encoded into one-hot vectors (that represent categorical latent variable). VAE without MI maximization.}
\label{fig:hists_nomi}
\end{figure}

\end{document}